\documentclass[letterpaper, 10 pt, conference]{ieeeconf} 
\IEEEoverridecommandlockouts

\usepackage{setspace}
\usepackage{caption}
\usepackage{subcaption}
\usepackage{cite}
\usepackage{amsmath,amssymb,amsfonts}
\usepackage{amsthm}
\usepackage{graphicx}
\usepackage{textcomp}
\usepackage{xcolor}
\usepackage{lipsum}
\usepackage{arydshln}
\usepackage{mwe}
\usepackage[justification=centering]{caption}
\def\BibTeX{{\rm B\kern-.05em{\sc i\kern-.025em b}\kern-.08em
    T\kern-.1667em\lower.7ex\hbox{E}\kern-.125emX}}
    
\usepackage{lipsum}
\usepackage{mathtools}
\usepackage{cuted}
\usepackage{dsfont}
\usepackage[linesnumbered,ruled]{algorithm2e} 

\usepackage{pgf,tikz}
\usetikzlibrary{positioning}
\usepackage{pgfplots}
\pgfplotsset{compat = newest}
\pgfplotsset{every x tick label/.append style={font=\footnotesize, yshift=0.5ex}}
\pgfplotsset{every y tick label/.append style={font=\footnotesize, xshift=0.5ex}}

\newtheorem{theorem}{Theorem}
\newtheorem{proposition}{Proposition}

\newtheorem{example}{Example}

\newtheorem{lemma}{Lemma}

\theoremstyle{remark}
\newtheorem{remark}{Remark}

\allowdisplaybreaks

\DeclarePairedDelimiter\parentheses{\lparen}{\rparen}
\newcommand{\vect}[1]{\operatorname{vec} \parentheses*{#1}}
\newcommand{\rank}[1]{\operatorname{rank} \parentheses*{#1}}
\newcommand{\spanv}[1]{\operatorname{span} \parentheses* {#1}}

\newcommand{\diag}[1]{\mbox{diag}\left(#1\right)\,}

\newcommand{\colspan}[1]{\operatorname{colspace}\left(#1\right)\,}

\newcommand{\kernel}[1]{\operatorname{kernel}\left(#1\right)\,}

\newcommand{\vectinv}[1]{\operatorname{vec}^{-1} \parentheses*{#1}}
\newcommand{\defdomP}{\mathbb{R}^{n\times k}}
\newcommand{\defdomQ}{\mathbb{R}^{m\times k}}

\newcommand{\target}{\mathcal{T}}

\newcommand{\comment}[1]{}

\newcommand{\SigBlkI}[1][~]{\begin{bmatrix}
    \bar\Sigma_{#1} & 0 & 0 \\ 0 & 0 & 0 \\ 0 & 0 & 0
\end{bmatrix}}

\newcommand{\SigBlkF}[1][~]{\begin{bmatrix}
    0 & 0 & 0 \\ 0 & \bar\Sigma_{#1} & 0 \\ 0 & 0 & 0
\end{bmatrix}}

\DeclareMathOperator*{\argmin}{\arg\!\min}

\definecolor{darkgreen}{HTML}{008000}
\definecolor{darkred}{HTML}{A52A2A}

\begin{document}

\title{\LARGE \bf On the ISS Property of the Gradient Flow for Single Hidden-Layer Neural Networks with Linear Activations\thanks{This work was partially supported by ONR Grant N00014-21-1-2431 and AFOSR Grant FA9550-21-1-0289.}
}

\author{\normalsize Arthur Castello B. de Oliveira$^{1}$, Milad Siami$^{1}$, and Eduardo D. Sontag $^{1,2}$
\thanks{$^{1}$Department of Electrical \& Computer Engineering,
Northeastern University, Boston, MA 02115 USA
	(e-mails: {\tt\small \{castello.a, m.siami, e.sontag\}@northeastern.edu}).}
\thanks{$^{2}$Departments of Bioengineering,
Northeastern University, Boston, MA 02115 USA.}%
}

\maketitle
\begin{abstract}

    Recent research in neural networks and machine learning suggests that using many more parameters than strictly required by the initial complexity of a regression problem can result in more accurate or faster-converging models -- contrary to classical statistical belief. This phenomenon, sometimes known as ``benign overfitting'', raises questions regarding in what other ways might overparameterization affect the properties of a learning problem. In this work, we investigate the effects of overfitting on the robustness of gradient-descent training when subject to uncertainty on the gradient estimation. This uncertainty arises naturally if the gradient is estimated from noisy data or directly measured. Our object of study is a linear neural network with a single, arbitrarily wide, hidden layer and an arbitrary number of inputs and outputs. In this paper we solve the problem for the case where the input and output of our neural-network are one-dimensional, deriving sufficient conditions for robustness of our system based on necessary and sufficient conditions for convergence in the undisturbed case. We then show that the general overparametrized formulation introduces a set of spurious equilibria which lay outside the set where the loss function is minimized, and discuss directions of future work that might extend our current results for more general formulations. 

\end{abstract}


\section{Introduction}

    Benign overfitting is a phenomenon observed when training large/deep neural networks \cite{belkin2019reconciling}. This observation, when first made, was disruptive because it challenged the traditional notion that overfitting a model always decreased its performance. Since then, many works have attempted to explain or understand this phenomenon \cite{belkin2018understand,belkin2019reconciling,sanyal2020benign,cao2022benign,bartlett2020benign,zhang2021understanding,frei2022benign,arora2019implicit,gunasekar2017implicit,soudry2018implicit} for different classes of system.

    Two common simplifying assumptions made when analysing benign overfitting on neural networks are that of linear activation functions and a single hidden layer, which together make the problem equivalent to a linear regression. In \cite{bartlett2020benign} the authors derive conditions for benign overfitting to happen in linear regression problems. In \cite{arora2019implicit,gunasekar2017implicit} the authors discuss how the gradient descent on overparameterized linear regressions tends to prefer solutions with good generalization properties. These works show that this simplified version of the problem is not only an important first step for gaining insights about the general case, but is by itself this still a rich and complex problem.
    
    Moreover, overparameterization in linear regression problems is also known to potentially accelerate the training process \cite{chen2020accelerating,arora2018optimization}, and in recent works \cite{tarmoun2021understanding,min2021explicit} the authors characterize the convergence rate of the system as a function of the initialization of the gradient flow dynamics. Their work shows that the more imbalanced (in a sense formally defined in the papers) the initial conditions are, the quicker the system converges to an equilibrium. Furthermore, they show in \cite{min2021explicit} that the gradient flow for overparameterized linear regressions converges at least as quickly as the non overparameterized case. This is a very interesting result and naturally raises the question of what is the disadvantage of using overparameterized formulations, if not only the accuracy might be increased, but the training time is actually potentially quicker as well, despite the higher number of parameters.
    
    We then look at the robustness trade-off from adopting overparameterized formulations. Many works \cite{tsipras2018robustness,ilyas2019adversarial,javanmard2020precise,ribeiro2023overparameterized, min2021curious, yin2019rademacher} evaluate the post training performance when the input is subject to adversarial disturbances. This became a very active area of research once it was noticed in \cite{szegedy2013intriguing} that by applying imperceptible noises, one could completely fool image recognition networks, despite their high training accuracy. The works in the area focus on adapting the training process in order to maximize the robustness of the network to adversarial attacks while still maintaining satisfying performance.

    Other papers analyse the robustness of the gradient flow as a tool for minimizing arbitrary functions. In \cite{scaman2020robustness} the authors analyse the convergence of the stochastic gradient descent as a function of the probability distribution of the gradient noise. In \cite{sontag2022remarks} one of the coauthors showed that the gradient flow is ISS when the estimation of the gradient is uncertain as long as the loss function satisfies some conditions. This establishes a sense of robustness for this class of systems when no overparametrization is considered, and is a motivation for this paper to study whether this property is maintained, and to what degree, once an overparameterized formulation is considered.
    
    Many works in the literature analyse the behaviour of linear neural networks when submitted to some gradient dynamics \cite{Chitour2023geometric,kawaguchi2016deep,baldi1989neural,monzon2006local,panageas2016gradient,du2018algorithmic,schaeffer2020extending,eftekhari2020training}. Collection the many results in the literature for gradient systems and neural networks presented for different assumptions on the system allow us to conclude that for linear neural networks with a single hidden layer: all local minimum of the cost function are global minimum; all non local minimum critical points are strict saddles (the Hessian has at least one negative eigenvalue); all solutions converge to a critical point of the cost function; and for almost every initial condition the solutions converge to the global minimum. All of these results give a complete qualitative understanding of the behaviour of our solutions and allow us to conclude stochastic guarantees for the problem under consideration.
    
    In this work we take steps towards quantitatively characterizing the local behaviour of solutions of the gradient flow dynamics for linear neural networks with a single hidden layer, and study its robustness properties. We will formulate the problem of uncertain gradient flow for the general case, but in this preliminary paper we work out in detail only the scalar/vector case. Understanding in depth this simplified instance of the problem gives important insight on what to look for, and how to understand, the general case, and we discuss similarities and differences between the scalar/vector and general cases. We then provide a complete characterization of the local behavior of our system around the origin and the target set in terms of the SVDs of the parameter matrices at those critical points.

\section{Motivation and Formulation}

        \subsection{Preliminary Definitions}

            Along this paper we use $I$ to denote the identity matrix, that is a square matrix whose all elements are zero except for the ones in its main diagonal, which are one. If we want  to emphasise the dimension of the identity matrix we write $I_n$ where $I\in\mathbb{R}^n$, otherwise if the dimension is not indicated it is clear from context. Similarly, we define $e_i$ as the $i$-th elementary vector which is the $i$-th column of $I$ for the dimension implicit in the context. The matrix $E_{ij}=e_ie_j^\top$ is called an elementary matrix and has all elements zero except for the one in row $i$ and column $j$, which is one (notice that $E_{ij}$ does not need to be square). Let $\|\cdot\|_F:\mathbb{R}^{m\times n}\rightarrow \mathbb{R}_+$ and $\|\cdot\|_2:\mathbb{R}^n\rightarrow\mathbb{R}_+$ denote the Frobenius and $\ell_2$ norms for arbitrary matrix and vector spaces respectively.

            Let $\vect{\cdot}:\mathbb{R}^{n\times m}\rightarrow\mathbb{R}^{nm}$ be the vectorization operator which concatenates the columns of its input matrix into a single column vector. Notice that $\vect{\cdot}$ is bijective and thus admits inverse, denoted in this papers as $\vectinv{\cdot}$. Furthermore, let $\otimes$ denote the Kronecker product (which is commutative and bilinear), then for three arbitrary matrices $A$, $B$ and $C$ of matching dimensions, the following well known identity is used freely during some derivations of this paper
            \begin{equation}
                \label{eq:kronprodprop}
                \vect{ABC} = \left(C^\top\otimes A\right) \cdot \vect{B}.
            \end{equation}
        \subsection{Overparametrized Linear Regression}
            Given a set of $\ell$ paired sampled inputs $x = \{x_i\}_{i=1}^\ell$ and outputs $y = \{y_i\}_{i=1}^\ell$, where $x_i\in\mathbb{R}^n$ and $y_i\in\mathbb{R}^m$, the \emph{linear regression problem} can be expressed as the following optimization problem
            \begin{equation}
                \begin{aligned}
                    \min_{\Theta \in \mathbb{R}^{n\times m}} \quad & \frac{1}{2}\|Y-\Theta^\top X\|_F^2,
                \end{aligned}
            \end{equation}
            where $X$ and $Y$ are $n\times \ell$ and $m\times \ell$ real matrices whose $i$-th columns are $x_i$ and $y_i$, respectively.
            
            If we assume a rich enough dataset ($\ell>\max\{m,n\}$ and that $X$ is a full rank matrix) possibly with additive Gaussian noise, the unique solution to this problem can be obtained as follows:
            \begin{equation}
                \Theta^* = (YX^\dagger)^\top.
            \end{equation}
            
            In general, an optimization problem is said to be overparametrized if the number of parameters in its search-space is larger than the minimum number necessary to solve it. For the purposes of this paper, we define the overparameterized problem under consideration as follows:
            \begin{equation}
                \label{eq:linreg}
                \begin{aligned}
                    \argmin_{P\in\defdomP, Q\in\defdomQ} \quad & \frac{1}{2}\|Y-QP^\top X\|_F^2,
                \end{aligned}
            \end{equation}
            where $k\geq n,m$. One can verify that $P$ and $Q$ solve \eqref{eq:linreg} if and only if they also solve the following matrix factorization problem:
            \begin{equation}
                \label{eq:probformmatfac}
                \begin{aligned}
                    \argmin_{P\in\defdomP, Q\in\defdomQ} \quad & \frac{1}{2}\|\bar{Y}-PQ^\top\|_F^2,
                \end{aligned}
            \end{equation}
            where $\bar{Y} = \Theta^* = (YX^\dagger)^\top$, albeit at a different minimum value. One possible method for solving the matrix factorization problem \eqref{eq:probformmatfac} is the use of a gradient flow for the dynamics of the parameters, however the resulting dynamics can be shown to have multiple \emph{spurious equilibria}, that is, equilibrium points of the dynamics that do not lie in the target set defined by $\target := \{(P,Q) ~|~ \bar{Y}=PQ^\top\}$.
        
        \subsection{The Gradient Flow Dynamics}
            
            To impose a gradient flow dynamics for the parameters, let us define the loss function as follows:
            \begin{equation}
                \label{eq:lossfunc}
                \mathcal{L}(P,Q) = \frac{1}{2}\|\bar{Y}-PQ^\top\|_F^2.
            \end{equation}
            
            Then, as derived in \cite{min2021explicit}, we impose the following dynamics for the parameters $P$ and $Q$
            \begin{equation}
                \begin{split}
                    \dot{P} = -\nabla_P\mathcal{L}(P,Q) = (\bar{Y}-PQ^\top)Q, \\
                    \dot{Q} = -\nabla_Q\mathcal{L}(P,Q) = (\bar{Y}-PQ^\top)^\top P,
                \end{split}
            \end{equation}
            or equivalently
            \begin{equation}
                \label{eq:ssgradflow}
                \small
                \dot{Z} = \begin{bmatrix} 
                    \dot{P} \\ \dot{Q}
                \end{bmatrix} = \begin{bmatrix} 
                    (\bar{Y}-PQ^\top)Q \\ (\bar{Y}-PQ^\top)^\top P
                \end{bmatrix} = \begin{bmatrix}
                        f_P(P,Q) \\ f_Q(P,Q)
                \end{bmatrix} = f_Z(Z).
            \end{equation}
            
            Often, however, the gradient value used to enforce the dynamics is an estimation of its true value and has an uncertainty associated with it. To model this uncertainty we add two disturbance terms on the dynamics as below
            \begin{equation}
                \label{eq:ssgradflowpert}
                \small
                \dot{Z} = \begin{bmatrix} 
                    \dot{P} \\ \dot{Q}
                \end{bmatrix} = \begin{bmatrix} 
                    (\bar{Y}-PQ^\top)Q \\ (\bar{Y}-PQ^\top)^\top P
                \end{bmatrix} + \begin{bmatrix}
                        U \\ V
                \end{bmatrix} = f_Z(Z)+\begin{bmatrix}
                        U \\ V
                \end{bmatrix},
            \end{equation}
            where $U\in\mathbb{R}^{n\times k}$ and $V\in\mathbb{R}^{m\times k}$. In the next section we explore a candidate Lyapunov function as means to characterize the stability of our system as a function of the magnitude of our disturbances and of our initialization.
        
        \subsection{The Loss Function as a Candidate Lyapunov Function}
        
            A natural choice for a candidate Lyapunov function is our loss function. By definition, $\mathcal{L}(P,Q)>0$ whenever $P,Q$ are not on the target set $\bar{Y}=PQ^\top$. Furthermore, one can compute an upperbound on the time derivative of the loss function under gradient flow as follows
            \begin{equation}
                \label{eq:Ldotupbnd}
                \begin{split}
                    \dot{\mathcal{L}}(P,Q,U,V) &= \left\langle \nabla \mathcal{L}, \begin{bmatrix}
                        \dot P \\ \dot Q
                    \end{bmatrix}\right\rangle \\&= \left\langle \nabla \mathcal{L}, -\nabla \mathcal{L}+\begin{bmatrix}
                        U \\ V
                    \end{bmatrix}\right\rangle \\ &= -\|\nabla \mathcal{L}\|_F^2 + \left\langle \nabla \mathcal{L}, \begin{bmatrix}
                        U \\ V
                    \end{bmatrix}\right\rangle \\&\leq -\|\nabla \mathcal{L}\|_F^2+\|\nabla\mathcal{L}\|_F\left\| \begin{bmatrix}
                        U \\ V
                    \end{bmatrix}\right\|_F .
                \end{split}
            \end{equation}
            
With this, we can establish the following theorem:
            
            \begin{theorem}
            \label{thm:LdotUB}
            
                The time derivative of the loss function \eqref{eq:lossfunc} can be upperbounded as follows:
                \begin{equation}
                    \small
                    \label{eq:DissipIneq}
                    \dot{\mathcal{L}} \leq - \mathcal{L}(P,Q)\cdot\left(\sigma_{\min}^2(Q) + \sigma_{\min}^2(P)\right) + \frac{1}{2}\left\| \begin{bmatrix}
                            U \\ V
                        \end{bmatrix}\right\|^2_F.
                \end{equation}
            \end{theorem}
           
            \begin{proof}
    
                First, for two arbitrary matrices $A$ and $B$ consider the following lower-bound on the squared Frobenius-norm of their product:
                \begin{equation}
                    \label{eq:lowerboundfrobeniussingulavalue}
                    \begin{split}
                        &\|AB\|^2_F = \sum_i \|a_iB\|_2^2 \\ &\geq \sum_i\sigma_{\min}^2(B)\|a_i\|_2^2 = \sigma_{\min}^2(B)\|A\|_F^2,
                    \end{split}
                \end{equation}
                where $a_i$ is the $i$-th row of $A$ (and the Frobenious norm of the row vector $a_iB$ is the same as its Euclidean norm). Then we can prove the results from the Theorem through the following algebraic calculation
                {\small
                \begin{align}
                        \|\nabla \mathcal{L}\|_F^2 &= \|(\bar{Y}-PQ^\top)Q\|_F^2+\|(\bar{Y}-PQ^\top)^\top P)\|_F^2 \nonumber \\ &\geq \left\|\bar{Y}-PQ^\top\right\|^2_F\sigma_{\min}^2(Q) + \left\|\bar{Y}-PQ^\top\right\|_F^2\sigma_{\min}^2(P)\nonumber\\&=\left\|\bar{Y}-PQ^\top\right\|_F^2(\sigma_{\min}^2(Q) + \sigma_{\min}^2(P))\\ &= 2\mathcal{L}(P,Q)\cdot(\sigma_{\min}^2(Q) + \sigma_{\min}^2(P)),\nonumber
                \end{align}}
                which when applied to \eqref{eq:Ldotupbnd} results in
                {\small
                \begin{align}
                        \dot{\mathcal{L}} &\leq -\|\nabla \mathcal{L}\|_F^2+\|\nabla\mathcal{L}\|_F\left\| \begin{bmatrix}
                        U \\ V
                    \end{bmatrix}\right\|_F\nonumber \\ &\leq -\frac{1}{2}\|\nabla \mathcal{L}\|_F^2+\frac{1}{2}\left\| \begin{bmatrix}
                        U \\ V
                    \end{bmatrix}\right\|_F^2 
                    \\ &\leq - \mathcal{L}(P,Q)\cdot(\sigma_{\min}^2(Q) + \sigma_{\min}^2(P)) + \frac{1}{2}\left\| \begin{bmatrix}
                            U \\ V
                        \end{bmatrix}\right\|^2_F,\nonumber
                    \end{align}}
                completing this proof.
            \end{proof}

            

            This Theorem gives us a sufficient criteria for assuring convergence of our system to the target set, however it depends on us being able to lower-bound, a priori, the singular values of our parameter matrices along a trajectory. We then look into ways of lower bounding the quantity $\sigma_{\min}^2(Q)+\sigma_{\min}^2(P)$ as a function of our initialization in order to properly characterize the ISS property of this system for the simplified case where $n=m=1$. To obtain that lower-bound we study the undisturbed case and identify necessary and sufficient conditions for its convergence to the target set, which we use as guidelines in order to obtain a bound on the maximum admissible disturbance signal.
        
    \section{The Scalar and Vector Cases ($n=m=1, k\geq 1$)}
        \label{sec:scalarnvector}

        We assume along this section that $n=m=1$, that the scalar $\bar Y$ is nonegative (all results are still valid if otherwise, but the characterization of stable and unstable sets swap), and that $P$ and $Q$ are row vectors in general (in the particular case where $k=1$ they are scalars). For this simplified version of the problem, one can verify that the undisturbed dynamics of the parameters of the system is a nonlinear reparametrization of a linear dynamics, that is
        \begin{equation}
            \label{eq:vecgradflw}
            \normalsize
            \begin{bmatrix}
                    \dot{P} \\ \dot{Q}
            \end{bmatrix} = (\bar{Y}-PQ^\top)\begin{bmatrix}
                0 & 1 \\ 1 & 0
            \end{bmatrix}\begin{bmatrix}
                    P \\ Q
            \end{bmatrix},
        \end{equation}
        since $\bar{Y}-PQ^\top$ is a scalar function. This means that the trajectories on the state-space will look the same as the linear system $\dot{P} = Q$ and $\dot{Q} = P$ with an extra stable set whenever $\bar{Y}-PQ^\top=0$ (our target set) and a change in direction if $\bar{Y}-PQ^\top<0$. For the scalar case, we can draw the phase plane of our system, as in Fig. \ref{fig:scalarpp}. To formalize this conjecture, we linearize the system around the origin, which results in
        \begin{figure}
            \centering
            \begin{tikzpicture}[scale=1.2]
                \begin{axis}[
                    xmin = -3, xmax = 3,
                    ymin = -3, ymax = 3,
                    zmin = 0, zmax = 1,
                    axis equal image,
                    view = {0}{90},
                    xtick = {-3,-2,...,3},
                    ytick = {-3,-2,...,3},
                    xlabel={\footnotesize $P$},
                    ylabel={\footnotesize $Q$},
                ]
                    \addplot3[
                        quiver = {
                            u = {(1-x*y)*y/(1.2*abs(1-x*y)*sqrt(x^2+y^2))},
                            v = {(1-x*y)*x/(1.2*abs(1-x*y)*sqrt(x^2+y^2))},
                            scale arrows = 0.25,
                        },
                        -stealth,
                        domain = -3:3,
                        domain y = -3:3,
                        thin,
                        color=gray!50!white,
                    ] {0};
                    
                    \addplot[blue,domain=-3:-1/3] {1/(x)};
                    \addplot[blue,domain=1/3:3] {1/(x)};
                    \addplot[red,dashed,domain=-3:-1.5/3] {1.5/(x)};
                    \addplot[red,dashed,domain=-3:-1.5/9] {1.5/(3*x)};
                    \addplot[red,dashed,domain=1.5/3:3] {1.5/(x)};
                    \addplot[red,dashed,domain=1.5/9:3] {1.5/(3*x)};
                    \addplot[green!60!black,dashed,domain=-3:3] {0.5-x};
                    \addplot[green!60!black,dashed,domain=-3:3] {-0.5-x};
                    
                    \draw plot[black,thin,smooth] file {tikzplotfiles/Solution1_n3p3.txt};
                    \draw plot[black,thin,smooth] file {tikzplotfiles/Solution2_p3n3.txt};
                    \draw plot[black,thin,smooth] file {tikzplotfiles/Solution3_n01n01.txt};
                    \draw plot[black,thin,smooth] file {tikzplotfiles/Solution4_p01p01.txt};
                    \draw plot[black,thin,smooth] file {tikzplotfiles/Solution5_n3n3.txt};
                    \draw plot[black,thin,smooth] file {tikzplotfiles/Solution6_p3p3.txt};
                    \draw plot[black,thin,smooth] file {tikzplotfiles/Solution7_p2p3.txt};
                    \draw plot[black,thin,smooth] file {tikzplotfiles/Solution8_p3p15.txt};
                    \draw plot[black,thin,smooth] file {tikzplotfiles/Solution9_n3p05.txt};
                    \draw plot[black,thin,smooth] file {tikzplotfiles/Solution10_n1n3.txt};
                    \draw plot[black,thin,smooth] file {tikzplotfiles/Solution11_p2n3.txt};
                    \draw plot[black,thin,smooth] file {tikzplotfiles/Solution12_n12p3.txt};
                    \draw plot[black,thin,smooth] file {tikzplotfiles/Solution13_n3p24.txt};
                    \draw plot[black,thin,smooth] file {tikzplotfiles/Solution14_p3n17.txt};
                \end{axis}
            \end{tikzpicture}

            \caption{Phase plane for the scalar case of the gradient flow dynamics. The trajectories followed by our solutions are the same as the 2D saddle, except for the inclusion of a new stable set whenever our nonlinear reparametrization $(\bar{Y}-PQ)=0$ and a change in the direction of the trajectories when $(\bar{Y}-PQ)<0$. In the figure there are highlited a couple of different solutions for the system (in blue) as well as the borders of two possible invariant sets (black and pink) that guarantee for initial conditions in them that the system converges to the target set, given sufficiently bound disturbances $U$ and $V$.}
            \label{fig:scalarpp}
        \end{figure}
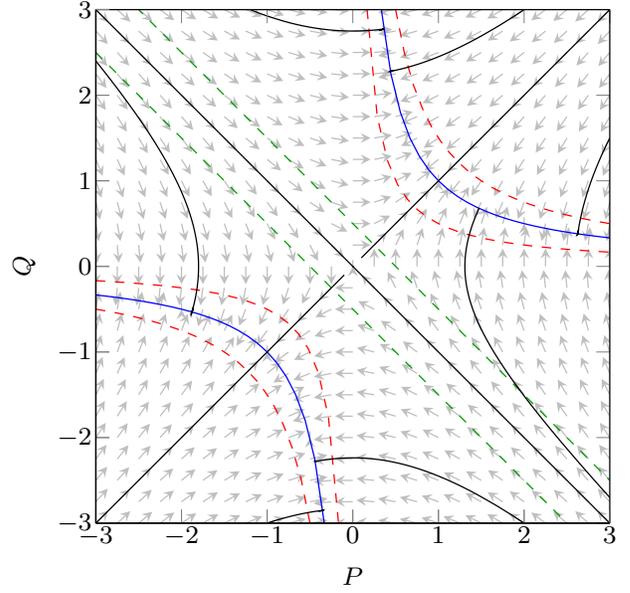
        \begin{equation}
            \normalsize
            \begin{bmatrix}
                    \dot{P} \\ \dot{Q}
            \end{bmatrix} = A_{\text{lin}}\left(P,Q\right),
        \end{equation}
        where $A_\text{lin}:\mathbb{R}^{2\times k}\rightarrow\mathbb{R}^{2\times k}$ is a linear operator on a matrix space. To write this in the familiar state space form we vectorize both sides of the equation, resulting in:
        \begin{equation}
            \small
            \textbf{vec}\begin{pmatrix}
                \dot{P} \\ \dot{Q}
            \end{pmatrix} = \left(I_{k}\otimes \begin{bmatrix}
                    0 & \bar{Y} \\ \bar{Y} & 0
            \end{bmatrix}\right)\textbf{vec}\begin{pmatrix} 
                P \\ Q
            \end{pmatrix} = \bar{A}_\text{lin}\textbf{vec}\begin{pmatrix} 
                P \\ Q
            \end{pmatrix}.
        \end{equation}
        
        One can verify that $\bar{A}_\text{lin}$ has eigenvalues $+\bar{Y}$ and $-\bar{Y}$ with multiplicity $k$, and that a orthogonal basis of eigenvectors associated with the positive (resp. negative) eigenvalues is given by $\{e_i\otimes [1, ~1]^\top\}_{i=1}^k$ (resp, $\{e_i\otimes [-1, ~1]^\top\}_{i=1}^k$). We then associate the linearization of our system around zero with its global (nonlinear) behavior in the following Lemma.

        \begin{proposition}
            For any initial condition $[P_0;Q_0]$ such that 
            $$\vect{\begin{bmatrix}
                P_0 \\ Q_0
            \end{bmatrix}}\in\mathcal{S}^-:=\mbox{span}\left(\left\{e_i\otimes \begin{bmatrix}
                -1 \\ 1
            \end{bmatrix}\right\}_{i=1}^k\right),$$
            the solution of \eqref{eq:vecgradflw} converges to the saddle point $[P(t);Q(t)]=0$. On the other hand, for every initial condition $[P_0;Q_0]$ such that $\vect{[P_0;Q_0]}\not\in\mathcal{S}^-$, the solution of \eqref{eq:vecgradflw} converges to the target set $\target$.
        \end{proposition}

        \begin{proof}
            %
            Let $P:\mathbb{R}_+\rightarrow\mathbb{R}^{1\times k}$ and $Q:\mathbb{R}_+\rightarrow\mathbb{R}^{1\times k}$ be such that $Z=[P;Q]:\mathbb{R}_+\rightarrow\mathbb{R}^{2\times k}$ is a solution of \eqref{eq:vecgradflw} and define $v(t):=\vect{Z(t)}$. Notice that we can write $v$ as a function of the basis of eigenvectors of the linearization of \eqref{eq:vecgradflw} around the origin as
            \begin{equation}
                v(t) = \sum_{i=1}^k a_i(t)\cdot e_i\otimes \begin{bmatrix}
                    1 \\ 1
                \end{bmatrix} + b_i(t)\cdot e_i\otimes \begin{bmatrix}
                    -1 \\ 1
                \end{bmatrix},
            \end{equation}
            where $a_i(t)$ and $b_i(t)$ are differentiable scalar functions of time, since the projection into a static basis of vectors is a time-invariant linear mapping. Define $a(t) = [a_1(t), ~a_2(t), ~\dots, ~a_k(t)]^\top$ and similarly for $b(t)$ (we will mostly omit any explicit time dependencies until the end of this proof whenever it is clear from context), and notice that by using the bilinearity and associativity properties of the Kronecker product we can write
            \begin{equation*}
                v = a\otimes\begin{bmatrix}
                    1 \\ 1
                \end{bmatrix} + b\otimes \begin{bmatrix}
                    -1 \\ 1
                \end{bmatrix}.
            \end{equation*}
            
            Applying the inverse vectorization operator on both sides gives
            \begin{align}
                \vectinv{v} &= \vectinv{a\otimes\begin{bmatrix}
                    1 \\ 1
                \end{bmatrix} + b\otimes \begin{bmatrix}
                    -1 \\ 1
                \end{bmatrix}}\nonumber \\ &= \vectinv{a\otimes\begin{bmatrix}
                    1 \\ 1
                \end{bmatrix}} + \vectinv{b\otimes \begin{bmatrix}
                    -1 \\ 1
                \end{bmatrix}},
            \end{align}
            which after applying \eqref{eq:kronprodprop} to both terms of the right hand side of the equation gives
            {\small
            \begin{align}
                \vectinv{v} = \begin{bmatrix}
                    P \\ Q
                \end{bmatrix} = \begin{bmatrix}
                    1 \\ 1
                \end{bmatrix} a^\top + \begin{bmatrix}
                    -1 \\ 1
                \end{bmatrix} b^\top = \begin{bmatrix}
                    a^\top - b^\top \\ a^\top + b^\top
                \end{bmatrix}.
            \end{align}}

            Applying this equation to \eqref{eq:vecgradflw} gives 
            \begin{equation}
                \begin{split}
                    \begin{bmatrix}
                        \dot a^\top - \dot b^\top \\ \dot a^\top + \dot b^\top
                    \end{bmatrix} = (\bar Y -\|a\|_2^2+\|b\|_2^2)\begin{bmatrix}
                        a^\top+b^\top \\ a^\top-b^\top
                    \end{bmatrix}.
                \end{split}
            \end{equation}
            
            To simplify our notation we define $\bar a := \|a\|_2^2$, $\bar b := \|b\|_2^2$, and $F(\bar{a},\bar{b}) := (\bar Y - \bar{a}+\bar{b})$. With this and after some manipulation we get
            \begin{equation}
                \begin{split}
                    \dot a &= F(\bar a,\bar b)a \\
                    \dot b &= -F(\bar a,\bar b)b,
                \end{split}
            \end{equation}
            which in turn implies
            \begin{equation}
                \begin{split}
                    \dot{\bar{a}} &= 2F(\bar a,\bar b)\bar{a} \\
                    \dot{\bar{b}} &= -2F(\bar a,\bar b)\bar{b}.
                \end{split}
            \end{equation}
            To understand the global behavior of the system, first notice that the sets of $P$ and $Q$ such that $F(\bar a(P,Q), \bar b(P,Q))>0$ and $F(\bar a(P,Q), \bar b(P,Q))<0$ are invariant. To prove that notice that if there exists a solution $Z = [P ; Q]$ of \eqref{eq:vecgradflw} such that at $t_1$, $F(\bar a(P(t_1),Q(t_1)), \bar b(P(t_1),Q(t_1))) <0$ and at $t_2>t_1$, $F(\bar a(P(t_2),Q(t_2)), \bar b(P(t_2),Q(t_2)))>0$ then there must exist $\bar t$$ \in $$(t_1,t_2)$ such that $F(\bar a(P(\bar t),Q(\bar t)), \bar b(P(\bar t),Q(\bar t)))=0$. However, that is a contradiction, since at any point such that $F(\bar a, \bar b)=0$, $\dot{\bar{a}} = 0$ and $\dot{\bar{b}} = 0$, and sincell the solutions of our ODE are unique.

            With that established, we can consider the behavior of our system in each of those sets independently. First consider $F(\bar a, \bar b)>0$. We need to treat two different cases for this set: one where $\bar a\equiv0$; and another where $a\neq 0$ for all $t$. First notice that if at any $t$, $\bar a(t)=0$ then it must be zero for all $t$, by a similar argument as the one we used to justify invariance of $F(\bar a, \bar b)>0$. Second, if $\bar a\equiv 0$ then the dynamics of $\bar b$ simplify to 
            \begin{equation}
                \dot{\bar{b}} = -2(\bar Y+\bar b)\bar b,
            \end{equation}
            which can easily be shown to always converge to zero (since $\bar b>0$) by checking the Lyapunov Function $V(\bar b) = \bar b^2$. Therefore, if $a\equiv 0$, then $F(\bar a, \bar b)>0$ and $b\rightarrow 0$. Furthermore, notice that $a$ is the vector obtained by projecting $v$ into the span of the unstable eigenvectors of our linearization, which means that $\bar a = 0 $$\iff$$ a=0 $$\iff $$v=\vect{[P;Q]}\in\mathcal{S}^-$. This, therefore, proves the first statement of the theorem, which is that for any solution starting in $\mathcal{S}^-$, the system converges to the saddle point. 

            For the case where $F(\bar a, \bar b)>0$ and $\bar a\neq 0$, notice that $\dot{\bar a} >0$ and $\dot{\bar b} \leq 0$ for all $t$. From this and the fact that $\bar b\geq 0$ we conclude that as $t\rightarrow \infty$, $\bar b\rightarrow \bar b_{ss}\geq 0$, since it is non-increasing. This, together with the fact that $F(\bar a,\bar b)>0$ for all $t$ implies that $\bar a$ is bounded above, since if it was not, then $F(\bar a, \bar b)$ would become negative at some time since $\bar b$ is non-increasing. Because $\bar a$ is bounded above and strictly increasing, it must converge to some $\bar a_{ss}>0$. Finally, because at $(\bar a_{ss}, \bar b_{ss})$ the system is at an equilibrium point, and $\bar a_{ss}$ is nonzero, that can only happen if $F(\bar a_{ss}, \bar b_{ss})=0$.

            At this point we bring attention to the fact that if $P=a^\top-b^\top$ and $Q = a^\top+b^\top$, then $\bar Y-PQ^\top = \bar Y-\bar a + \bar b = F(\bar a, \bar b)$. This means that converging to a point where $F(\bar a, \bar b)=0$ is equivalent to converging to a point in the target set.


            For the case where $F(\bar a, \bar b)<0$, first notice that $a\neq 0$ necessarily, since if $a=0$, then $\bar a=0$ which would imply that $F(0,\bar b) = \bar Y + \bar b$ could never be negative. The rest of the argument follows analogously as the case where $F(\bar a, \bar b)>0$ and $a\neq 0$ and is, thus, omitted due to space.

            With this we show that as long as $a\neq 0$, the system always converges to a point where $F(\bar a, \bar b)=0$ which is equivalent to say that it converges to a point in our target set.
        \end{proof}

        
        
        \begin{remark}
            For the scalar case, $\mathcal{S}^-$ and $\mathcal{S}^+$ (defined the same as $\mathcal{S}^-$ but for the unstable eigenvectors) correspond to the lines $P-Q=0$ and $P+Q=0$ respectively, as can be seen in Fig. \ref{fig:scalarpp}. The condition that $F(\bar a,\bar b)<0$ can be understood as being to the southwest of the lower hyperbola or to the northeast of the higher hyperbola, while $F(\bar a,\bar b)>0$ is the region in between the two hyperbolas. As mentioned before, the set where $[P;Q]$ is spanned by the eigenvectors of our linearization associated to the negative eigenvalues is the line $P+Q=0$ and is the only set in the state space that converges to the saddle point $[P;Q]=0$.             In Fig. \ref{fig:scalarpp}, the dashed green lines are defined by the equation $|P+Q|=k$ which intuitively measure the magnitude of our projection in $\mathcal{S}^+$, while the dashed red lines are defined by $PQ=k$ which in some sense measures the distance between a point and our target set along our vector field. While both $|P+Q|>k$ and $PQ>k$ can be shown to be forward invariant, we focus on the green lines for this section, in hopes of being able to more easily generalize this set for the general case in the future. 
        \end{remark}
        
        Considering the minimum distance to $\mathcal{S}^-$ (that is the norm of the projection into $\mathcal{S}^+$) as a possible invariant set for our problem gives
        {\small
        \begin{equation}
            \normalsize
            \left\|\textbf{vec}^{-1}\left(\left(I_k\otimes \begin{bmatrix}
                    1 \\ 1
            \end{bmatrix}\right)^\top \cdot \vect{\begin{bmatrix} 
                P \\ Q
            \end{bmatrix}}\right)\right\|_2^2 = \|P+Q\|_2^2.
        \end{equation}}
        
        For some $\alpha>0$, define the set
        \begin{equation}
            \normalsize
            \mathcal{R}_\alpha = \{P,Q\in\mathbb{R}^n ~|~ \|P+ Q\|_2^2\geq \alpha^2\},
        \end{equation}
        for which we can state the following theorem:
        
        \begin{theorem}
            \label{thm:InvSet}
            For any $\alpha\in [0,2\sqrt{\bar{Y}})$, the set $\mathcal{R}_\alpha$ is forward-invariant under the gradient flow dynamics if
            \begin{equation}
            \normalsize
                 \|U\|_2+\|V\|_2\leq \frac{1}{\sqrt{2}}|\alpha|(\bar{Y}-\frac{\alpha^2}{4}).
            \end{equation}
            
            Moreover, if $(P,Q)\in\mathcal{R}_\alpha$, then  $PP^\top + QQ^\top = \sigma(P)^2+\sigma(Q)^2 \geq \alpha^2/2$. 
    
      \end{theorem}
      
      \begin{proof}
            A point in the border of $\mathcal{R}_\alpha$ is such that $\|P+Q\|_2^2=\alpha^2$. At such point, the vector field under uncertain gradient flow dynamics is given by:
            {
            \begin{align}
                    &~\frac{d}{dt}\|P(t)+Q(t)\|_2^2 \nonumber \\ =&~\frac{d}{dt}\left((P(t)+Q(t))(P(t)+Q(t))^\top\right) \nonumber \\ =& ~ \frac{d}{dt}(P P^\top + 2P Q^\top + Q Q^\top) \nonumber \\ =&~ 2(\bar{Y}-P Q^\top)(P P^\top + 2P Q^\top + Q Q^\top)\\&~+2(P U^\top + P V^\top + Q U^\top + Q V^\top) \nonumber \\ =&~ 2(\bar{Y}-P Q^\top)(P P^\top + 2P Q^\top + Q Q^\top)\nonumber \\&~+2(P+Q)(U+V)^\top \nonumber \\ \geq&~ 2(\bar{Y}-P Q^\top)\|P+Q\|_2^2-2\|P+Q\|_2\|U+V\|_2 \nonumber.
            \end{align}}
            
            For $\mathcal{R}_\alpha$ to be invariant, we need the vector field at any point of its border ''not to point to outside of the set``. For $\mathcal{R}_\alpha$ this means that $d/dt \|P+Q\|_2^2>0$ at any point of the border. Since $\|P+Q\|_2$ is always nonegative, it is easy to see that the above lowerbound is nonegative if
            \begin{equation}
                \label{eq:ubndnonsym}
                \|U+V\|_2\leq(\bar{Y}-P Q^\top)\|P+Q\|_2.
            \end{equation}
            
            Notice that we can write an upper-bound for the left side of \eqref{eq:ubndnonsym} as
            \begin{equation}
                \|U+V\|_2\leq\sqrt{2\|U\|_2^2+2\|V\|_2^2} \leq \sqrt{2}(\|U\|_2+\|V\|_2),
            \end{equation}
            and a lower-bound for the right side as
            \begin{equation}
                (\bar{Y}-P Q^\top)\|P+Q\|_2 \geq \alpha(\bar{Y}-\frac{\alpha^2}{4}),
            \end{equation}
            since at any point at the border of $\mathcal{R}_\alpha$, $\|P+Q\|_2^2=\alpha^2\geq 4P Q^\top$, we can state the following sufficient condition for invariance of the set:
            \begin{equation}
                \|U\|_2+\|V\|_2\leq \frac{1}{\sqrt{2}}\alpha(\bar{Y}-\frac{\alpha^2}{4}).
            \end{equation}

            The second statement of the theorem comes from simply computing $\alpha^2=\|P+Q\|_2^2\leq 2\|P\|_2^2+2\|Q\|_2^2$.
        \end{proof}
        
        This theorem allows us to rewrite the previous lower-bound dissipation inequality, for a solution initialized in $\mathcal{R}_\alpha$ for some $\alpha\in [0, 2\sqrt{\bar{Y}})$, as
        \begin{equation}
            \normalsize
            \dot{\mathcal{L}}(P,Q,U,V)\leq -\mathcal{L}(P,Q)\frac{\alpha^2}{2}+\frac{1}{2}\left\| \begin{bmatrix}
                    U \\ V
                \end{bmatrix}\right\|_F^2,
        \end{equation}
        which is strictly negative until we reach the target set, assuming our disturbances respect the bound in Theorem \ref{thm:InvSet} and that our disturbance signals satisfy the conditions on Theorem \ref{thm:InvSet}.
        
    \section{The General Case}

        For the general case, where $m$ and $n$ are arbitrary, we have to deal with a matrix ODE for the dynamics of the parameters. One immediate problem from this is the existence of spurious equilibria besides the origin. For formally showing this, we first state the following intermediate Lemma \ref{lemma:aux1} followed by Theorem \ref{lemma:eqs} which characterize the system's equilibria.

        \begin{lemma}
            \label{lemma:aux1}
            Given two matrices $A\in\mathbb{R}^{p\times o}$ and $B\in\mathbb{R}^{q\times o}$ for $p,q,o\in\mathbb{N}$ with $q\geq o$, the following two statements are equivalent
            
            \begin{enumerate}
                \item $AB^\top=0$;
                \item {There exist orthogonal matrices $\Psi_A$, $\Phi$, and $\Psi_B$, and rectangular diagonal matrices with non-negative diagonal elements
            $\Sigma_A$ and $\Sigma_B$, such that
            \begin{equation}
                \label{eq:auxlemA}
                \begin{split}
                    &A = \Psi_A\Sigma_A\Phi^\top
                \end{split}
            \end{equation}
            and
            \begin{equation}
                \label{eq:auxlemB}
                B = \Psi_B\Sigma_B\Phi^\top
            \end{equation}
            are SVDs of $A$ and $B$, and $\Sigma_A\Sigma_B^\top =0$.}
            \end{enumerate}
            
            Furthermore, in $2)$ we can write $\Sigma_A$ and $\Sigma_B$ as

            \begin{equation*}
                \Sigma_A = \SigBlkI[A]
            \end{equation*}
            and
            \begin{equation*}
                \Sigma_B =\SigBlkF[B]
            \end{equation*}
            where $\bar\Sigma_A$ and $\bar\Sigma_B$ are diagonal matrices whose main diagonal elements are the nonzero singular values of $A$ and $B$ respectively.
        \end{lemma}

        \begin{proof}
            The proof of $2)$$\Rightarrow$$1)$ follows immediately from writing
            
            $$AB^\top = \Psi_A\Sigma_A\Phi^\top\Phi\Sigma_B^\top \Psi_B^\top = \Psi_A\Sigma_A\Sigma_B^\top\Psi_B^\top = 0.$$
            
            To prove $1)$$\Rightarrow$$2)$, let $a = \rank{A}$ and $b=\rank{B}$, then write $A$ and $B$ as a sum of rank one matrices as follows
            \begin{equation}
                \label{eq:svdA}
                A = \sum_{i=1}^a \psi_{i,A}\phi_{i,A}^\top\sigma_{i,A}
            \end{equation}
            and
            \begin{equation}
                \label{eq:svdB}
                B = \sum_{i=1}^b\psi_{i,B}\phi_{i,B}^\top \sigma_{i,B},
            \end{equation}
            where $\psi_{i,A/B}$ and $\phi_{i,A/B}$ are left and right singular vectors of $A$ and $B$ associated with nonzero singular values, and $\sigma_{i,A/B}$ are the nonzero singular values of $A$ and $B$.

            Notice that $AB^\top=0$ if and only if any vector in the span of the columns of $B^\top$ belongs to the kernel of $A$. Therefore, since $\colspan{B^\top} = \spanv{\{\phi_{i,B}\}_{i=1}^b}$ and $\kernel{A}$ is the orthogonal complement of $\spanv{\{\phi_{i,A}\}_{i=1}^a}$, which is equivalent to $\phi_{i,A}^\top\phi_{j,B}=0$ for all $(i,j)\in \{1,\dots,a\}\times\{1,\dots,b\}$, i.e. $\spanv{\{\phi_{i,A}\}_{i=1}^a}$$\perp$$\spanv{\{\phi_{i,B}\}_{i=1}^b}$. 

            
            \comment{To see this, define $c_{ij}=\sigma_{i,A}\sigma_{j,B}\phi_{i,A}^\top \phi_{j,B}$, and assume that for some $i_0$ and $j_0$, $c_{i_0j_0}\neq 0$, and define $\mathcal{J} = \{1,\dots,a\}\times\{1,\dots,b\}-\{(i_0,j_0)\}$. Then $AB^\top = 0$ only if
            \begin{align*}
                -c_{i_0j_0}\psi_{i_0,A}\psi_{j_0,B}^\top &= \sum_{(i,j)\in\mathcal{J}}c_{ij}\psi_{i,A}\psi_{j_0,B}^\top \\
                -c_{i_0j_0}\bar\Psi_A E_{i_0,j_0} \bar\Psi_B^\top &= \sum_{(i,j)\in\mathcal{J}}c_{ij}\bar\Psi_AE_{ij}\bar\Psi_B^\top \\
                -c_{i_0j_0} E_{i_0,j_0}  &= \sum_{(i,j)\in\mathcal{J}}c_{ij}E_{ij},
            \end{align*}
            where $\bar\Psi_A$ and $\bar\Psi_B$ are the matrices whose columns are the vectors $\psi_{i,A}$ and $\psi_{i,B}$, respectively. However the set of elementary matrices is linearly independent, and as such there can never exist such set of $c_{ij}$.}
            
            This immediately implies that $a+b\leq o$. Let $\bar o = o-a-b\geq 0$, then consider any set of orthonormal vectors $\{\phi_{i,0}\}_{i=1}^{\bar o}$ \comment{such that 
            \begin{equation*}
                \spanv{\{\phi_{i,0}\}_{i=1}^{\bar o}}\perp \spanv{\{\phi_{i,A}\}_{i=1}^a\cup \{\phi_{i,B}\}_{i=1}^b}
            \end{equation*}
            and
            \begin{equation*}
                \spanv{\{\phi_{i,0}\}_{i=1}^{\bar o}\cup\{\phi_{i,B}\}_{i=1}^b\cup\{\phi_{i,A}\}_{i=1}^a} = \mathbb{R}^o.
            \end{equation*}

            The set $\{\phi_{i,0}\}_{i=1}^{\bar o}$ always exists since $\bar o\geq 0$ with it being the empty set if $\bar o = 0$, just} that completes an orthonormal basis of $\mathbb{R}^o$ from $\{\phi_{i,B}\}_{i=1}^b\cup\{\phi_{i,A}\}_{i=1}^a$. 
            
            Define $\Phi_A$, $\Phi_B$ and $\Phi_0$ as the matrices whose columns are the vectors of $\{\phi_{i,A}\}_{i=1}^a$, $\{\phi_{i,B}\}_{i=1}^b$, and $\{\phi_{i,0}\}_{i=1}^{\bar o}$, respectively. With these definitions, we can express the matrix $\Phi$ as
            \begin{equation*}
                \small
                \Phi = \begin{bmatrix}
                    \Phi_A & \Phi_B & \Phi_0
                \end{bmatrix}.
            \end{equation*}

            Next, consider the sets $\{\psi_{i,A}\}_{i=1}^{p}$ and $\{\psi_{i,B}\}_{i=1}^{q}$ constructed such that the first $a$ and $b$ vectors of each set satisfy \eqref{eq:svdA} and \eqref{eq:svdB}, respectively. The remaining vectors are simply chosen to complete orthonormal bases for $\mathbb{R}^p$ and $\mathbb{R}^q$, respectively. Then build $\Psi_A$ as the matrix whose columns are the vectors of $\{\psi_{i,A}\}_{i=1}^p$ in the same order as in the set, and $\Psi_B$ as the matrix whose columns with indices from $a+1$ to $a+b$ are the first $b$ vectors of $\{\psi_{i,B}\}_{i=1}^{q}$ and whose remaining columns are the remaining vectors in the set in no particular order (here is where the assumption that $q>o$ is used, and the extra care on the order of the columns is necessary to make sure they match the order or the columns of $\Phi$).

            Finally, let $\Sigma_A\in\mathbb{R}^{p\times o}$ and $\Sigma_B\in\mathbb{R}^{q\times o}$ be rectangular diagonal matrices with the first $a$ elements of the main diagonal of $A$ being the elements of  $\{\sigma_{i,A}\}_{i=1}^{a}$, and the elements of the main diagonal of $\Sigma_B$ of indices from $a+1$ to $a+b$ being the elements of $\{\sigma_{i,B}\}_{i=1}^{b}$ (all remaining main diagonal elements are zero).
            
            With all matrices built as indicated, we can verify that
            \begin{equation*}
                \Psi_A\Sigma_A\Phi = \sum_{i=1}^a \psi_{i,A}\phi_{i,A}^\top\sigma_{i,A} = A
            \end{equation*}
            and
            \begin{equation*}
                \Psi_B\Sigma_B\Phi = \sum_{i=1}^b\psi_{i,B}\phi_{i,B}^\top \sigma_{i,B} = B,
            \end{equation*}
            which means that they are valid SVDs of $A$ and $B$, respectively. Furthermore, $\Sigma_A\Sigma_B^\top=0$ follows directly from their construction.
        \end{proof}


        \begin{theorem}
            \label{lemma:eqs}
            For the dynamics given by \eqref{eq:ssgradflow}, and an arbitrary set of parameters $ P$ and $ Q$, the following are equivalent
            
            \begin{enumerate}
                \item $Z = [ P;  Q]$ is an equilibrium point of $f_Z$, that is $(\bar{Y}- P Q^\top) Q=0$ and $(\bar{Y}- P Q^\top)^\top P=0$;
                \item 
                (a) There exist an SVD of $\bar{Y}- P Q^\top=\Psi{\Sigma}\Phi^\top$, and orthogonal matrices $\Gamma_{ P}$ and $\Gamma_{ Q}$ such that $\Psi\Sigma_{ P}\Gamma_{ P}^\top$ and $\Phi\Sigma_{ Q}\Gamma_{ Q}^\top$ are SVDs of $ P$ and $ Q$; and (b) $\Sigma\Sigma_Q=0$ and $\Sigma^\top\Sigma_P=0$.
            \end{enumerate}
        \end{theorem}

        \begin{proof}
            To prove that $2)$$\Rightarrow$$ 1)$ we simply compute $f_Z( P, Q)$ for a $( P, Q)$ that satisfies the properties in $2$ and verify that it is equal to zero. That is
            {
            \begin{align}
                    f_Z( P, Q) =& \begin{bmatrix} 
                        (\bar{Y}-PQ^\top)Q \\ (\bar{Y}-PQ^\top)^\top P
                    \end{bmatrix} \nonumber\\=& \begin{bmatrix} 
                        \Psi{\Sigma}\Phi^\top \Phi\Sigma_{ Q}\Gamma_{ Q}^\top \\ \Phi{\Sigma}^\top\Psi^\top \Psi\Sigma_{ P}\Gamma_{ P}^\top
                    \end{bmatrix} \\=& \begin{bmatrix} 
                        \Psi{\Sigma}\Sigma_{ Q}\Gamma_{ Q}^\top \\ \Phi{\Sigma}^\top\Sigma_{ P}\Gamma_{ P}^\top
                    \end{bmatrix} = \begin{bmatrix}
                        0 \\ 0
                    \end{bmatrix}.\nonumber
            \end{align}}
            
            To prove that $1)$$\Rightarrow$$2)$, apply Lemma \ref{lemma:aux1} with $A = (\bar{Y}-PQ^\top)$ and $B = Q^\top$, which implies that $o=m\leq k=q$, since $f_P(P,Q) = (\bar{Y}-PQ^\top)Q = 0$, which allows us to write $ Q$ and $(\bar{Y}- P Q^\top)$ as follows:
            {
            \begin{align}
                \label{eq:lmauxintoQ}
                    (\bar{Y}-PQ^\top) = \Psi_Q \underbrace{\SigBlkI[1]}_{\Sigma_1} \Phi^\top \\ Q = \Phi \underbrace{\SigBlkF[Q]}_{\Sigma_Q} \Gamma_Q.\nonumber
            \end{align}}
            
            Similarly, applying Lemma \ref{lemma:aux1} with $A = (\bar{Y}-PQ^\top)^\top$ and $B = P$ gives 
            {
            \begin{align}
                \label{eq:lmauxintoP}
                    (\bar{Y}-PQ^\top) = \Psi \underbrace{\SigBlkI[2]}_{\Sigma_2} \Phi_P^\top \\ P = \Psi \underbrace{\SigBlkF[P]}_{\Sigma_P}\Gamma_P.\nonumber
            \end{align}}
            
            Notice that even if $\bar\Sigma_1\neq\bar\Sigma_2$, they must still have the same diagonal elements, albeit possibly in a different order. Changing the order of the elements of $\bar\Sigma_1$ and $\bar\Sigma_2$ so they match means swapping the columns of $\Psi$, $\Phi$, $\Psi_Q$ and $\Phi_P$, but we can also swap the singular vectors corresponding to the kernels of $P^\top$ and $Q^\top$ such that the results from Lemma \ref{lemma:aux1} still hold. As such we can assume without loss of generality that $\bar\Sigma_1=\bar\Sigma_2=\bar\Sigma$. Notice that this is enough to prove that $1\rightarrow 2b$, since $(\bar Y-PQ^\top)Q=\Psi_Q\Sigma\Phi^\top\Phi\Sigma_Q\Gamma_Q=0\iff \Sigma\Sigma_Q=0$ (and similarly for $\Sigma^\top\Sigma_P=0$).
            
            Next we write
            {
            \begin{align}
                    \begin{bmatrix}
                    \Psi_{1,Q} & \Psi_{2,Q} & \Psi_{3,Q}
                \end{bmatrix}&\SigBlkI\begin{bmatrix}
                    \Phi_1^\top \\ \Phi_2^\top \\ \Phi_3^\top
                \end{bmatrix} \\ = \begin{bmatrix}
                    \Psi_1 & \Psi_2 & \Psi_3
                \end{bmatrix}&\SigBlkI\begin{bmatrix}
                    \Phi_{1,P}^\top \\ \Phi_{2,P}^\top \\ \Phi_{3,P}^\top
                \end{bmatrix},\nonumber
            \end{align}}
            which implies that 
            {
            \begin{align}
                \label{eq:lemequality}
                \Psi_{1,Q}\bar\Sigma\Phi_1^\top &= \Psi_1\bar\Sigma\Phi_{1,P}^\top.
            \end{align}}

            If we impose, for example, $\Psi_{1,Q} = \Psi_1$, then we must also impose $\Phi = \Phi_{1,P}$. This leaves the SVD of $P$ intact, but changes part of the SVD of $Q$. To show it still satisfies Lemma \ref{lemma:aux1}, consider

            \begin{equation*}
                Q = \begin{bmatrix}
                    \Phi_{1,P} & \Phi_2 & \Phi_3
                \end{bmatrix}\SigBlkF[P]\begin{bmatrix}
                    \Gamma_{1,P}^\top \\ \Gamma_{2,P}^\top \\ \Gamma_{3,P}^\top
                \end{bmatrix} = \Phi_2\bar\Sigma_{P}\Gamma_{2,P}^\top
            \end{equation*}
            which shows that the above is still a valid SVD of $Q$ as long as the span of the columns of $[\Phi_{1,P}, \Phi_3]$ is equal to the left Kernel of $Q$. Indeed, we know that $\colspan{[\Phi_{1}, \Phi_3]} = \kernel{Q^\top}$, and since \eqref{eq:lemequality} shows that $\colspan{\Phi_1} = \colspan{\Phi_{1,P}}$ (by contradiction) then we can conclude that $\colspan{[\Phi_{1,P}, \Phi_3]} = \kernel{Q^\top}$. We have, therefore, established that we can always match the singular vectors associated with the nonzero singular values of $\bar Y-PQ^\top$ for the SVDs in \eqref{eq:lmauxintoQ} and \eqref{eq:lmauxintoP} and still satisfy both conditions on Lemma \ref{lemma:aux1}.


            Next, notice that we can simply pick $\Psi_{3,Q} = \Psi_3$, $\Phi_{3,P} = \Phi_3$, since all matrices are related to the intersection of the kernels and their choice is arbitrary as long as they compose an orthonormal basis of $\kernel{\bar Y-PQ^\top}\cap\kernel{Q}$ and of $\kernel{\bar Y^\top-QP^\top}\cap\kernel{P}$ respectively. 
            
            For the remaining matrices, $\Psi_2$ and $\Phi_2$ are imposed by the SVDs of $P$ and $Q$ respectively, and as such cannot be changed arbitrarily. We can, however, freely change the columns of $\Psi_{2,Q}$ (resp. $\Phi_{2,P}$) as long as when composed with the columns of $\Psi_3$ (resp. $\Phi_3$) they form a basis of the kernel of $(\bar Y - PQ^\top)^\top$ (resp. $\bar Y - PQ^\top$). Therefore we can select $\Psi_{2,Q}$ (resp. $\Phi_{2,P}$) to be equal to $\Psi_2$ (resp $\Phi_2$) without any loss of generality, completing the proof.
        \end{proof}



        \begin{remark}
            In \cite{tarmoun2021understanding} the authors characterize the equilibria of the system for the symmetric case (when $P=Q$ and $\bar Y$ is positive definite). Reinterpreting statement $2)$ of Theorem \ref{lemma:eqs} for this case gives that $\bar Y - PP^\top = \Psi\Sigma\Psi^\top$ and $P = \Psi\Sigma_P\Gamma_P$, which in turn imposes that $\bar Y = \Psi\Sigma_Y\Psi^\top$. Using these SVDs, the equation $(\bar Y-PP^\top)P = 0$ becomes $(\Sigma_Y-\Sigma_P\Sigma_P^\top)\Sigma_P = 0$ which holds if for every $i$ either $\sigma_{i,Y} = \sigma_{i,P}^2$ or $\sigma_{i,P}=0$, which recovers their original result.
        \end{remark}
        
        The presence of multiple equilibria in our dynamics means that the approach done for the vector case does not immediately translates to the general case. Despite that, we look for ways to extend the intuition from the vector case where we can predict the global behaviour of the system from its linearization. To facilitate the local analysis of our system around our equilibria, we define the vectorized dynamics and show its equivalence to the original problem. After that, we derive the linearized dynamics
        
        \subsection{The Vectorized Dynamics}
        
            To compute the gradient of the vector field for the general case we will work with the dynamics on the vectorized space. For that goal, let $p = \vect{P}$ and $q=\vect{Q}$ and define the gradient dynamics for these new parameters as $\dot{p} = -\nabla_p\mathcal{L}(p,q)$ and $\dot{q}=-\nabla_q\mathcal{L}(p,q)$. Also, notice that 
            
            \begin{align}
                \mathcal{L} &= \frac{1}{2}\|\bar{Y}-PQ^\top\|_F^2 \nonumber\\&= \frac{1}{2}\|\vect{\bar Y}-\vect{PQ^\top}\|_2^2 \nonumber\\&= \frac{1}{2}\|\vect{\bar Y}-Q\otimes I\cdot p\|_2^2 \nonumber\\&= \frac{1}{2}\|\vect{\bar Y}-K_{mn}P\otimes I \cdot q\|_2^2
            \end{align}
            where for any two arbitrary positive integers $p$ and $q$, $K_{pq}\in\mathbb{R}^{pq\times pq}$ is the permutation matrix that solves for any matrix $M\in\mathbb{R}^{p\times q}$ the equation $\vect{M} = K_{pq}\cdot\vect{M^\top}$. From here we compute
            {\small
            \begin{align}
                    \dot{p} &= -\nabla_p\mathcal{L} \nonumber\\&= Q^\top\otimes I\cdot \vect{\bar Y-PQ^\top} \nonumber\\&= \vect{(\bar Y-PQ^\top)Q} \nonumber\\&= \vect{\dot{P}} \\
                    \dot{q} &= P^\top\otimes I\cdot K_{mn}\vect{\bar Y-Q^\top} \nonumber\\&= \vect{(\bar Y-PQ^\top)^\top P} \nonumber\\&= \vect{\dot{Q}},
            \end{align}}
            that is, as defined, the dynamics of the vectorized parameters $P$ and $Q$ is the same as the vectorization of their original dynamics. This means we can work with the system in the vectorized space and from there take conclusions about the original system's behaviour. Let
            {\small
            \begin{align}
                \label{eq:fzdef}
                    \dot{z} &= \begin{bmatrix}
                    \vect{\dot{P}} \\ \vect{\dot{Q}}
                \end{bmatrix} \nonumber\\&= \begin{bmatrix}
                    \dot{p} \\ \dot{q}
                \end{bmatrix} \nonumber\\&= \begin{bmatrix} 
                    \vect{(\bar{Y}-PQ^\top)Q} \\ \vect{(\bar{Y}-PQ^\top)^\top P}
                \end{bmatrix} \nonumber\\&= \begin{bmatrix}
                    f_p(p,q) \\ f_q(p,q)
                \end{bmatrix} \nonumber\\&= f_z(z),
            \end{align}}
            and assume ${P}$ and ${Q}$ are an equilibrium of $f_Z$ (that is, ${p}$ and ${q}$ are an equilibrium of $f_z$), then one can show that:
            {\small
            \begin{align}
                f_p({p}+\Delta p, {q}) &= \vect{(\bar{Y}-{P}{Q}^\top){Q}}-\vect{\Delta P {Q}^\top {Q}} \nonumber\\&= f_p({p}, {q}) \underbrace{-Q^\top Q\otimes I }_{\nabla_p f_p({p},{q})}\Delta p
            \end{align}}
            {\small
            \begin{align}
                    f_p({p},{q}+\Delta{q}) =& \vect{(\bar{Y}-{P}{Q}^\top){Q}}+\vect{(\bar{Y}-{P}{Q}^\top)\Delta Q} \nonumber\\&- \vect{{P}\Delta Q^\top {Q}} - \vect{{P}\Delta Q^\top \Delta Q} \nonumber\\ =& f_p({p},{q}) \nonumber\\&+ \underbrace{(I\otimes (\bar{Y}-{P}{Q}^\top)-{Q}^\top\otimes{P}K_{mk})}_{\nabla_qf_p({p},{q})}\Delta q \nonumber\\&+ \text{h.o.t.}
            \end{align}}
            {\small
            \begin{align}
                    f_q({p}+\Delta{p},{q}) =& f_q({p},{q}) \nonumber\\&+ \underbrace{(I\otimes (\bar{Y}-{P}{Q}^\top)^\top-{P}^\top\otimes{Q}K_{nk})}_{\nabla_pf_q({p},{q})}\Delta p \nonumber\\&+ \text{h.o.t.}
            \end{align}}
            {\small
            \begin{equation}
                f_q({p},{q}+\Delta{q}) = f_q({p},{q})\underbrace{-P^\top P\otimes I}_{\nabla_qf_q({p},{q})}\Delta q 
            \end{equation}}
            
            Notice that for any pair of matrices $A\in \mathbb{R}^{n\times k}$ and $B\in\mathbb{R}^{m\times k}$, $K_{nk}$ and $K_{mk}$ also satisfies $A^\top \otimes B = K_{mk} B \otimes A^\top K_{nk}$. From this property we can see that, as we expected, the Hessian is symmetric, that is:
            
            \begin{equation}
                \nabla_qf_p({p},{q}) = \nabla_pf_q({p},{q})^\top
            \end{equation}
            since
            \begin{align}
                \nabla_Pf_q^\top({p},{q}) &= I\otimes (\bar{Y}-{P}{Q}^\top)-K_{nk}{P}\otimes{Q}^\top \nonumber\\&= I\otimes (\bar{Y}-{P}{Q}^\top)-{Q}^\top\otimes{P}K_{mk} \nonumber\\&= \nabla_qf_p({p},{q}).
            \end{align}
            
            We then write the linearization of the vectorized dynamics around an arbitrary equilibrium point as the Hessian:
            
            \comment{\begin{equation}
                \label{eq:genlindef}
                \nabla f_z = \begin{bmatrix}
                    -Q^\top Q\otimes I & I\otimes (\bar{Y}-{P}{Q}^\top)-{Q}^\top\otimes{P}K_{mk} \\ I\otimes (\bar{Y}-{P}{Q}^\top)^\top-{P}^\top\otimes{Q}K_{nk} & -P^\top P\otimes I
                \end{bmatrix}
            \end{equation}}
            \begin{equation}
                \label{eq:genlindef}
                \nabla f_z = \begin{bmatrix}
                    \nabla_pf_p(p,q) & \nabla_qf_p(p,q) \\ \nabla_pf_q(p,q) & \nabla_qf_q(p,q)
                \end{bmatrix}.
            \end{equation}
            {
            \subsubsection*{About the partitions of $P$, $Q$, and $\bar Y-PQ^\top$}
                
                Define $\bar p = \rank{P}$, $\bar q = \rank{Q}$ and $\ell = \rank{\bar Y-PQ^\top} = m-\bar q$. If $n>m$ and $\rank{\bar Y}=m$ then necessarily $\kernel{Q^\top}\cap\kernel{\bar Y-PQ^\top}=\emptyset$, sice if that was not the case and there were some nonzero $v$ in the kernel of $\bar Y-PQ^\top$ such that $Q^\top v=0$ then $0=(\bar Y-PQ^\top)v = \bar Yv-PQ^\top v = \bar Yv$ which contradicts the assumption that $\bar Y$ is full rank. Similarly, any vector in the span of the columns of $\Psi_3$ must be in the left kernel of $\bar Y$, by a similar argument. 
                
                Also, one can see that $\rank{\Gamma_{2,P}^\top\Gamma_{2,Q}}=\bar q$ for if there were some vector $v$ such that $\Gamma_{2,P}^\top\Gamma_{2,Q}v=0$ then, define $w = \Gamma_{2,Q}\bar\Sigma_Q^{-2}v$, and write
                
                \begin{align*}
                    (\bar Y-PQ^\top)Qw &= (\bar Y-PQ^\top)\Phi_2 \bar\Sigma_Q^{-1}v \\ 
                        &= \bar Y\Phi_2\bar\Sigma_{Q}^{-1}v - \Psi_2\bar\Sigma_P\Gamma_{2,P}^\top\Gamma_{2,Q}v \\
                        &= \bar Y\Phi_2\bar\Sigma_{Q}^{-1}v,
                \end{align*}
                which can only be zero if $\bar Y$ is rank deficient, breaking the assumption. \comment{A direct consequence of this is that $\kernel{P}\subseteq\kernel{Q}$, which means that if $\Gamma_k$ is a matrix with $k-\bar p$ orthogonal columns that compose a basis of the kernel of $P$, then $Q\Gamma_k=0$.}
                
                
                With all that established, given an arbitrary equilibrium $[P;Q]\neq 0$, we can simplify the SVDs from Theorem \ref{lemma:eqs} as follows
                
                \begin{align}
                    P &= \begin{bmatrix}
                        \Psi_1 & \Psi_2 & \Psi_3
                    \end{bmatrix}\begin{bmatrix}
                        0 & 0 & 0 \\ 0 & \bar \Sigma_{P} & 0 \\ 0 & 0 & 0
                    \end{bmatrix}\begin{bmatrix}
                    \Gamma_{1,P}^\top \\ \Gamma_{2,P}^\top \\ \Gamma_{3,P}^\top
                    \end{bmatrix} \\
                    Q &= \begin{bmatrix}
                        \Phi_1 & \Phi_2
                    \end{bmatrix}\begin{bmatrix}
                        0 & 0 & 0 \\ 0 & \bar\Sigma_Q & 0
                    \end{bmatrix}\begin{bmatrix}
                        \Gamma_{1,Q}^\top \\ \Gamma_{2,Q}^\top \\ \Gamma_{3,Q}^\top
                    \end{bmatrix}\\
                    \bar Y-PQ^\top &= \begin{bmatrix}
                        \Psi_1 & \Psi_2 & \Psi_3
                    \end{bmatrix}\begin{bmatrix}
                        \bar\Sigma & 0 \\ 0 & 0 \\ 0 & 0
                    \end{bmatrix}\begin{bmatrix}
                        \Phi_1^\top \\ \Phi_2^\top
                    \end{bmatrix},
                \end{align}
                where $\Phi_1$ and $\Psi_1$ both have $\ell$ columns, $\Phi_2$ has $\bar q$ columns, $\Psi_2$ has $\bar p$ columns, and $\Psi_3$ has $n-m-\bar p+\bar q$ columns. The matrices $\Gamma_{1,Q}$, $\Gamma_{1,P}$, $\Gamma_{2,Q}$ and $\Gamma_{2,P}$ have the same number of columns as $\Phi_1$, $\Psi_1$, $\Phi_2$ and $\Psi_2$ respectively. Matrices $\Gamma_{3,Q}$ and $\Gamma_{3,P}$ have $k-m$ and $k-m-\bar p+\bar q$ columns, respectively. In our next section we derive the linearization of our system around any equilibria as characterized by Theorem \ref{lemma:eqs} and present their eigenvalues. Knowing exactly how these matrices are partitioned is useful since we will define the eigenvectors of our linearizations as functions of our SVDs. Notice also that if $[P;Q]=0$ then $\bar p = \bar q = 0$ and if $\bar Y = PQ^\top$ then $\ell =0$.
                %
                %
                %
                }
        \subsection{Linearization Around the Equilibria}
            \begin{lemma}
                The linearization of $f_z(z)$ given by \eqref{eq:fzdef}, around the equilibrium $[P;Q]=0$ is given by
                
                \begin{equation}
                    \label{eq:gradfz0}
                    \nabla f_z = \begin{bmatrix}
                        0 & I\otimes \bar{Y} \\ I\otimes \bar{Y}^\top & 0
                    \end{bmatrix},
                \end{equation}
                has as nonzero eigenvalues the set $\{\pm\bar\sigma_i\}_{i=1}^{\text{\emph{rank}}(\bar{Y})}$ each with multiplicity $k$ and, assuming $n>m$, as an orthogonal basis of eigenvectors the columns of
                
                \begin{equation}
                    V_{\nabla f_z} := \begin{bmatrix}
                        \Omega\otimes \Psi_1 & -\Omega \otimes \Psi_1 & \Omega \otimes \Psi_3\\ \Omega\otimes \Phi_1 & \Omega \otimes \Phi_1 & 0
                    \end{bmatrix}
                \end{equation}
                for $\Psi_1$ being the matrix whose columns contain the $m$ left singular vectors of $\bar{Y}$ associated with nonzero singular values, and any orthogonal matrix $\Omega\in\mathbb{R}^{k\times k}$.
            \end{lemma}
            
            \begin{proof}
                Notice that if $[P;Q]=0$ then $\Psi_2$ and $\Phi_2$ are empty and the SVD of $\bar Y-PQ^\top = \bar Y$ becomes simply
                
                \begin{align*}
                    \bar Y = \Psi\Sigma\Phi^\top = \begin{bmatrix}
                        \Psi_1 & \Psi_3
                    \end{bmatrix}\begin{bmatrix}
                        \bar\Sigma \\ 0
                    \end{bmatrix}\begin{bmatrix}
                        \Phi_1^\top
                    \end{bmatrix}.
                \end{align*}
                
                Also, notice that for this case, $\Sigma = \Sigma_{\bar Y}$. Rewrite $\nabla f_z$ from \eqref{eq:gradfz0} as follows:
                
                \begin{equation}
                    \small
                    \nabla f_z = \begin{bmatrix}
                        0 & I\otimes(\Psi\Sigma_{\bar Y}\Phi^\top) \\ I\otimes(\Phi\Sigma_{\bar Y}^\top\Psi^\top) & 0
                    \end{bmatrix}
                \end{equation}
                and multiply it from the right by $V_{\nabla f_z}$ as below:
                
                \begin{equation}
                    \small
                    \begin{split}
                        \nabla f_z \cdot V_{\nabla f_z} &= \begin{bmatrix}
                            \Omega\otimes\Psi\Sigma_{\bar{Y}} & \Omega\otimes\Psi\Sigma_{\bar{Y}} & 0 \\ \Omega\otimes\Phi\Sigma_{\bar{Y}}^\top  & -\Omega\otimes\Phi\Sigma_{\bar{Y}}^\top  & 0
                        \end{bmatrix} \\ &= \begin{bmatrix}
                        \Omega\otimes \Psi_1 & -\Omega \otimes \Psi_1 & \Omega \otimes \Psi_3\\ \Omega\otimes \Phi & \Omega \otimes \Phi & 0
                    \end{bmatrix} \\&~~~~~~~~~~~~\times \begin{bmatrix}
                        I\otimes \bar\Sigma_{\bar Y} & 0 & 0 \\ 0 & -I\otimes\bar\Sigma_{\bar Y} & 0 \\ 0 & 0 & 0
                    \end{bmatrix}.
                    \end{split}
                \end{equation}
                
                Notice that since $\Omega$ is an arbitrary orthogonal matrix in $\mathbb{R}^{k\times k}$, the given matrix of eigenvectors has $mk + mk + (n-m)k = (n+m)k$ orthonormal columns, which means it is a complete base of eigenvectors.
            \end{proof}
            \begin{lemma}
                \label{thm:TargetSetLinearization}
                The linearization of $f_z(z)$, given by \eqref{eq:fzdef}, around our target set $\target:=\{[P;Q]\in\mathbb{R}^{(n+m)\times k}~|~\bar{Y}=PQ^\top\}$ is given by
                
                \begin{equation}
                    \label{eq:lintargsetdef}
                    \nabla f_z = -\begin{bmatrix}
                    Q^\top Q\otimes I & {Q}^\top\otimes{P}K_{mk} \\ {P}^\top\otimes{Q}K_{nk} & P^\top P\otimes I
                    \end{bmatrix},
                \end{equation}
                has as eigenvalues the diagonal elements of
                
                \begin{equation}
                    \small
                        \Lambda_{\nabla f_z} = -\begin{bmatrix}
                        \bar\Sigma_Q^2\otimes I + I\otimes \bar\Sigma_P^2 & 0 & 0 & 0 & 0\\ 0 & \bar\Sigma_Q^2\otimes I & 0 & 0 & 0 \\ 0 & 0 & 0 & 0 & 0 \\ 0 & 0 & 0 & 0 & 0 \\ 0 & 0 & 0 & 0 & 0
                    \end{bmatrix},
                \end{equation}
                and as an orthogonal basis of eigenvectors the columns of 
                
                \begin{equation}
                    \small
                    V_{\nabla{f_z}} := \begin{bmatrix}
                            V_1 & V_2 & V_3 & V_4 & V_5
                \end{bmatrix},
                \end{equation}
                with 
                \begin{align*}
                    V_1 &= \begin{bmatrix}
                        \left(\Gamma_{2,Q}\bar \Sigma_Q\right)\otimes\Psi_2 \\ K_{mk}\cdot \Phi_2\otimes \left(\Gamma_{2,P}\bar\Sigma_P^\top\right)
                    \end{bmatrix}\in\mathbb{R}^{(m+n)k\times m\bar{p}} \\ 
                    V_2 &= \begin{bmatrix}
                        \Gamma_{2,Q}\otimes \Psi_3 \\ 0
                    \end{bmatrix} \in\mathbb{R}^{(m+n)k\times m(n-\bar{p})}\\
                    V_3 &= \begin{bmatrix}
                        -\Gamma_{2,Q} \otimes \left(\Psi_2\bar\Sigma_P\right) \\ K_{mk}\cdot \left(\Phi_1\bar\Sigma_Q\right)\otimes \Gamma_{2,P}
                    \end{bmatrix}\in\mathbb{R}^{(m+n)k\times m\bar{p}} \\ 
                    V_4 &= \begin{bmatrix}
                        \Gamma_{3,Q}\otimes \Psi \\ 0
                    \end{bmatrix}\in\mathbb{R}^{(m+n)k\times (k-m)n} \\
                    V_5 &= \begin{bmatrix}
                        0 \\ K_{mk}\cdot\Phi\otimes\Gamma_{3,P}
                    \end{bmatrix}\in\mathbb{R}^{(m+n)k\times m(k-\bar{p})}
                \end{align*}
                where $K_{nk}$ is the permutation matrix that solves for $A\in\mathbb{R}^{n\times k}$ the equation $\vect{A^\top} = K_{nk}\vect{A}$.
            \end{lemma}
            
            \begin{proof}
                First, notice that since we are in $\target$, $\ell=0$ which means that we can write $P$ and $Q$ as
                
                \begin{align*}
                    P &= \begin{bmatrix}
                        \Psi_2 & \Psi_3
                    \end{bmatrix}\begin{bmatrix}
                        \bar\Sigma_P & 0 \\ 0 & 0 
                    \end{bmatrix}\begin{bmatrix}
                        \Gamma_{2,P}^\top \\ \Gamma_{3,P}^\top
                    \end{bmatrix} \\
                    Q &= \begin{bmatrix}
                        \Phi_2
                    \end{bmatrix}\begin{bmatrix}
                        \bar\Sigma_Q & 0
                    \end{bmatrix}\begin{bmatrix}
                        \Gamma_{2,Q}^\top \\ \Gamma_{3,Q}^\top 
                    \end{bmatrix}
                \end{align*}
                
                Then, equation \eqref{eq:lintargsetdef} is obtained by simply combining \eqref{eq:genlindef} and $\bar{Y}=PQ^\top$. We then proceed to obtain its form as a function of some SVDs of its parameters parameters as follows:
                
                \begin{align}
                    Q^\top Q\otimes I &= \Gamma_Q \otimes \Omega_Q \cdot \Sigma_Q^T\Sigma_Q\otimes I \cdot \Gamma_Q^\top \otimes \Omega_Q^\top \\
                    {Q}^\top\otimes{P}K_{mk} &= \Gamma_Q \otimes \Psi \cdot \Sigma_Q^\top \otimes \Sigma_P \cdot \Phi^\top \otimes \Gamma_P^\top \cdot K_{mk} \\
                    {P}^\top\otimes{Q}K_{nk} &= \Gamma_P \otimes \Phi \cdot \Sigma_P^\top \otimes \Sigma_Q \cdot \Psi^\top \otimes \Gamma_Q^\top \cdot K_{nk} \\
                    P^\top P\otimes I &= \Gamma_P \otimes \Omega_P \cdot \Sigma_P^T\Sigma_P\otimes I \cdot \Gamma_P^\top \otimes \Omega_P^\top
                \end{align}
                
                \comment{\begin{align}
                    \nabla f_z &= -\begin{bmatrix}
                        Q^\top Q\otimes I & {Q}^\top\otimes{P}K_{mk} \\ {P}^\top\otimes{Q}K_{nk} & P^\top P\otimes I
                    \end{bmatrix}\nonumber \\&= -\begin{bmatrix}
                        \Gamma_Q \otimes \Omega_Q \cdot \Sigma_Q^T\Sigma_Q\otimes I \cdot \Gamma_Q^\top \otimes \Omega_Q^\top & \Gamma_Q \otimes \Psi \cdot \Sigma_Q^\top \otimes \Sigma_P \cdot \Phi^\top \otimes \Gamma_P^\top \cdot K_{mk} \\ \Gamma_P \otimes \Phi \cdot \Sigma_P^\top \otimes \Sigma_Q \cdot \Psi^\top \otimes \Gamma_Q^\top \cdot K_{nk} & \Gamma_P \otimes \Omega_P \cdot \Sigma_P^T\Sigma_P\otimes I \cdot \Gamma_P^\top \otimes \Omega_P^\top
                    \end{bmatrix}
                \end{align}}
            
                where $\Omega_P$ and $\Omega_Q$ are arbitrary orthogonal matrices, which we select to be $\Omega_P = \Phi$ and $\Omega_Q = \Psi$. One can, then, verify that
                
                \begin{align*}
                    \nabla f_z\cdot V_1 &= 
                        V_1\cdot (\bar\Sigma_Q^2\otimes I + I\otimes \bar\Sigma_P^2) \\
                    \nabla f_z\cdot V_2 &= 
                        V_2\cdot (\bar\Sigma_Q^2\otimes I) \\
                    \nabla f_z\cdot V_3 &= 0 \\
                    \nabla f_z\cdot V_4 &= 0 \\
                    \nabla f_z\cdot V_5 &= 0.
                \end{align*}
                Furthermore, $V_{\nabla f_z}$ has $(n+m)k$ orthogonal columns and as such is a basis of eigenvectors of our linearization.
        \end{proof}
        
        \begin{remark}
            Notice that our target set is formed by $mn$ constraints and as such has dimension $(m+n)k-mn$. The given set of eigenvectors has $mn$ eigenvectors associated with strickly negative eigenvalues and $(m+n)k-mn$ eigenvectors associated with zero eigenvalues, proving local convergence of $\target$.
        \end{remark}

        An interesting consequence of Lemma \ref{thm:TargetSetLinearization} is the effect of imbalance on the local convergence rate of our trajectories. To observe that, consider a point $[P;Q]$ in our target set such that $P=\Psi\Sigma_P\Gamma^\top$ and $Q=\Phi\Sigma_Q\Gamma^\top$ and define $T(\xi) = \xi I_k$. Then, it is easy to observe that any point of the form $[PT ; QT^{-1}]$ is also in the target set, however, the nonzero eigenvalues or our linearization around $[PT;QT^{-1}]$ are given by $(1/\xi)\bar\Sigma_Q^2\otimes I + I\otimes \xi\bar\Sigma_P^2$ and $(1/\xi)\bar\Sigma_Q^2\otimes I$ which grow unbounded as $\xi\rightarrow 0$. In some sense, this recovers locally around the target set the conclusions present in \cite{min2021explicit} about the level of imbalance accelerating the convergence.

\section{Conclusions and Future Works}

    In this paper we formulated the overparameterized linear regression problem as a matrix factorization and presented a dissipation-like inequality for the general problem when the parameters are trained through an uncertain gradient flow. The bound obtained, however, does not guarantee convergence to the target set $\target$ for any initial condition, which prompts the search for invariant subsets of the state space in which the system can be shown to be ISS.

    In this publication we focus on the solution of the problem for the case where the neural network has a single input and a single output. In this situation the parameter matrices reduce to vectors and the analysis is significantly simplified. We characterize the behavior of the system when training through exact gradient flow and formulate necessary and sufficient conditions for its convergence to the target set. We then use those conditions as a guideline to formulate sufficient conditions on the initialization of our system and on the maximum admissible disturbance on the estimation of the gradient that if satisfied guarantee that the system is ISS.

    We finish the paper with a discussion about the general case and some preliminary results. We show that in general the dynamics become significantly more complicated with the appearance of multiple sets of spurious equilibria, and provide the local behavior of the system through its linearization around boths the origin and the target set. While there are results in the literature that guarantee convergence to the target set for the general case, their extension to the disturbed case is not straightforward. 

    Current research being conducted by the authors for the general case indicate that despite its more complex dynamics, we can still predict the behaviour of the general case based on its linearization around the origin, similarly as to how we solve the problem on the vector case. We are currently looking into how we can use this knowledge to characterize regions of our state space where the ISS property is guaranteed in general.

    We can, however, conclude that by opting for an overparameterized formulation, our system ceases to be ISS for the whole state space when subject to gradient flow, contrary to the non overparameterized case, as shown in \cite{sontag2022remarks}. This indicates a possible trade-off on using an overparameterized formulation for performing linear regression, even if it is eventually shown that it can be circumvented by a knowledgeable choice of initial condition.

\bibliographystyle{IEEEtran}
\begin{spacing}{.8}
\bibliography{OverParmLinReg.bib}{}
\end{spacing}


\end{document}